\documentclass[runningheads]{llncs}

\usepackage{eccv}

\usepackage{eccvabbrv}

\usepackage{graphicx}
\usepackage{booktabs}

\usepackage[accsupp]{axessibility}  

\usepackage{hyperref}

\usepackage{orcidlink}
\usepackage{bbm}

\usepackage{pifont}
\newcommand{\Y}{\ding{51}}
\newcommand{\N}{\ding{55}}

\usepackage{verbatim}
\usepackage{marvosym}

\newtheorem {assumption}{Assumption}[section]

\usepackage{xcolor}
\usepackage{color,soul}
\definecolor{bronze}{rgb}{1,1,0.6}
\definecolor{silver}{rgb}{0.969,0.796,0.600}
\definecolor{gold}{rgb}{0.941,0.592,0.600}
\newcommand{\gold}[1]{\colorbox{gold}{{#1}}}
\newcommand{\silver}[1]{\colorbox{silver}{{#1}}}
\newcommand{\bronze}[1]{\colorbox{bronze}{{#1}}}

\newcommand{\nickname}{NeoMap}

\begin{document}

\title{\nickname: Training-free Novel-View Synthesis from Single Images and Videos 
} 

\titlerunning{\nickname}

\author{Jinxi Li\inst{1,2}\textsuperscript{\dag}\orcidlink{0000-0001-9683-1359} 
\and
Tianyi Zhang\inst{1,2}\textsuperscript{\dag}\orcidlink{0009-0008-6959-4544} 
\and
Yafei Yang\inst{1,2}\orcidlink{0000-0001-5278-7528}
\and
Zihui Zhang\inst{1,2}\orcidlink{0000-0002-1891-8253}
\and
Peng Huang\inst{1,2}\orcidlink{0009-0002-1229-3579}
\and
Koon Wing Macgyver Lin\textsuperscript{\ddag}\orcidlink{0009-0008-9276-2489}
\and
Bo Yang\inst{1,2}\textsuperscript{\ddag}\textsuperscript{\Letter}\orcidlink{0000-0002-2419-4140}
}

\authorrunning{J.~Li et al.}

\institute{Shenzhen Research Institute, The Hong Kong Polytechnic University \and vLAR Group, The Hong Kong Polytechnic University \\
\textsuperscript{\dag} equal contribution  \quad
\textsuperscript{\ddag} joint supervision \quad \textsuperscript{\Letter} corresponding author\\
\email{\{jinxi.li,tonax.zhang\}@connect.polyu.hk, bo.yang@polyu.edu.hk}}

\maketitle

\begin{abstract}
  We study the challenging problem of novel view video synthesis from single images or monocular videos. Existing methods, which operate under the assumption that pre-trained video models lack native novel view synthesis capability and enforce view alignment via camera conditioning, task-specific fine-tuning, or stepwise hard denoising guidance, often suffer from artifacts and compromised global scene consistency. In this paper, we introduce \textbf{\nickname{}}, a novel training-free framework designed to locate high-fidelity, view-consistent novel view solutions from general pre-trained video models. The key to our approach is the core insight that promising novel view solutions are inherently encoded within the natural video data manifold learned by pre-trained models, and the core challenge is simply to locate this optimal solution. We solve this via our core mechanism: convergent manifold alternating projection iterations that optimize the initial noise. Extensive experiments demonstrate that \nickname{} significantly outperforms all existing methods across 3 standard novel view synthesis benchmarks, including the challenging Tanks-and-Temples, LLFF and DAVIS datasets, achieving state-of-the-art generation fidelity and top-tier view consistency. Our code and data are available at \textcolor{red}{\href{https://github.com/vLAR-group/NeoMap}{https://github.com/vLAR-group/NeoMap}}
\end{abstract}

\section{Introduction}
\label{sec:intro}
Monocular novel view synthesis (NVS), which generates photorealistic unseen views from a single-view image or unconstrained monocular video, is a core task in computer vision with wide applications in AR/VR, 3D/4D content creation, and robotic perception. However, it remains an extremely challenging problem for traditional reconstruction-based methods \cite{kerbl3Dgaussians, mildenhall2020nerf} due to inherent severe occlusions, incomplete scene information, and depth ambiguity from monocular inputs. Recently, the rapid development of powerful large-scale video generation models \cite{blattmann2023Stable,chen2024videocrafter2,kong2024hunyuanvideo,yang2024cogvideox,wan2025}, which show a reasonable understanding of spatial-temporal relationships, has opened a promising new direction to address this long-standing challenge.

\begin{figure}[t]
\setlength{\abovecaptionskip}{ 2 pt}
\setlength{\belowcaptionskip}{ -2 pt}
\centering
\centerline{\includegraphics[width=1.\linewidth]{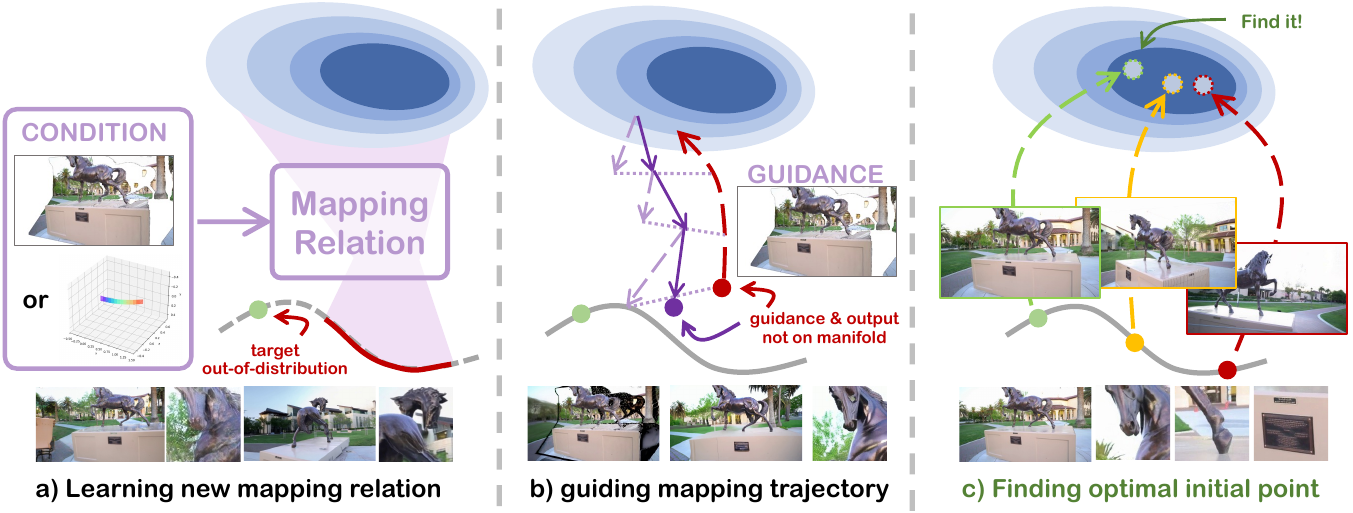}}
    \caption{Different NVS paradigms: a) Learning new mapping relation would shrink the output data manifold, leading to overfitting and OOD problem; b) Guiding generation trajectory with bad condition would lead the generation output away from data manifold; c) Different related output possibilities already exist in the data manifold. }
    \label{fig:paradigms}\vspace{-0.4cm}
\end{figure}

One line of work learns a new conditional video generation mapping by training models on 3D/4D data with camera pose annotations. Either injecting camera trajectories \cite{bai2025recammaster,vd3d,wang2024motionctrl,trajattn} or taking depth-based warped view images \cite{chen2025flexworld,yu2024viewcrafter,trajattn} as additional generation conditions, these methods essentially constrain the denoising process with pose-related signals. Although these methods show promising performance, they inevitably suffer from overfitting to limited camera trajectory distributions and depth estimation noise seen during training, leading to poor generalization to in-the-wild scenes and unseen camera motions. Another line of work avoids task-specific retraining via training-free guidance strategies \cite{you2025nvs,song2025worldforge}. They use depth-based warped images as reliable priors, and blend the latent features of priors with the generating latent to guide the generation denoising step-by-step. However, such hard step-wise constraints often disrupt the continuity of the native generation process of the pre-trained model, leading to severe artifacts including incomplete content, corrupted scene structure, and incorrect semantic understanding.

Despite their differences, as shown in Figure \ref{fig:paradigms}, both categories share a fundamental assumption: in order to obtain a valid and geometrically consistent NVS solution, \textit{the mapping relation between the noise space and the video latent space should be modified or redefined in a task-specific way}. In this work, we propose to challenge this assumption with a key insight: a high-quality NVS solution for a given input and target pose should inherently exist in the native output latent data manifold of a pre-trained video generation model, which has been trained on massive diverse real-world videos, whose output data manifold thus already covers rich 3D scenes and natural camera motions. Therefore, \textbf{instead of focusing on the mapping relation by altering the denoising trajectory or re-training the model, we propose to solve this problem simply by locating this valid target results in the output space}. 

Given a video generation model, it is extremely challenging to find a specific video in its intractable abstract output data manifold. Nevertheless, recent video diffusion methods can all be regarded as a deterministic mapping (ODE for flow-matching methods \cite{lipman2023flowmatching,liu2022flow} or PF-ODE for diffusion methods \cite{lu2022dpmsolver}) from regular gaussian noise space to video data manifold, which means every video in the output space can be defined by a point in the random noise space. Thus, our challenge of locating the valid video data point in the data manifold can be simply framed as searching for the optimal initial noise latent corresponding to the target valid result.

With this motivation, we propose an elegant pipeline to discover the exact initial noise latent that naturally decodes into a high-fidelity, geometrically consistent novel view video. Our framework, \textbf{\nickname{}} (\textbf{N}ois\textbf{E} initializati\textbf{O}n by \textbf{M}anifold \textbf{A}lternating \textbf{P}rojection), constructs this optimal initialization through three intuitive modules: 1) a \textbf{novel view prior module} establishes base geometric constraints via depth-based image warping, which inherently suffers from out-of-distribution artifacts and semantic emptiness in occlusions; 2) an \textbf{anchored manifold projection (AMP) module} addresses this by pushing the state toward the video data manifold to hallucinate plausible content in unobserved regions while adaptively anchoring reliable visible features; 3) a \textbf{pixel-constrained projection (PCP) module} rectifies the spatial bleeding caused by VAE latent blending by enforcing strict, immutable geometric boundaries back in the pixel space.

The core of our pipeline lies in the alternating projection between the AMP and PCP modules. By iteratively mapping the state between the unconstrained video data manifold and the strict NVS geometric constraints, \textbf{\nickname{}} mathematically guaranties that our initial noise converges to the optimal intersection of global semantic realism and precise 3D view alignment. Coupled with an auxiliary trajectory re-anchoring strategy to prevent integration drift during generation, this enables us to synthesize photorealistic videos without any task-specific finetuning or additional guidance. Our main contributions are:
\begin{itemize}
    \item We introduce a new training-free framework to solve the monocular novel view synthesis problem by finding a optimal initial noise, eliminating the need for any generation guidance or task-specific finetuning. 
    \item We propose to initialize the noise via manifold alternating projection, simultaneously achieving high fidelity and geometric consistency by ensuring that the generation output is located on the NVS constrained data manifold. 
    \item We demonstrate a significant improvement in visual quality and geometric consistency on three existing datasets for both static single-image and dynamic monocular video novel view synthesis.
\end{itemize}

\section{Related Work}
\label{sec:liter}
\textbf{Novel View Synthesis by Camera-Controlled Video Generation.}  Driven by the rapid advancement of foundational video generation models \cite{blattmann2023Stable,chen2024videocrafter2,kong2024hunyuanvideo,yang2024cogvideox,wan2025}, extensive research has explored conditional video generation to achieve precise camera control. A prominent line of work treats this as a conditional generation task. Early approaches in this domain \cite{animatediff, direct_a_video, wang2024motionctrl, cameractrl, CVD, ac3d} primarily focused on incorporating camera extrinsics directly into the generation process by fine-tuning diffusion models with specialized motion modules or parameter encoders to learn specific trajectory patterns. Building upon this, recent advancements introduce more explicit conditioning mechanisms to enhance both geometric alignment and visual quality. For instance, TrajectoryAttention \cite{trajattn} injects fine-grained motion control by aligning content along specified trajectories directly within the attention layers. TrajectoryCrafter \cite{trajcrafter} disentangles deterministic view transformations from stochastic generation by utilizing a dual-stream architecture conditioned on point cloud renders, while ReCamMaster \cite{bai2025recammaster} frames the problem as video-conditioned generative re-rendering to reproduce dynamic scenes under novel trajectories. 

\textbf{Novel View Synthesis by Warping and Inpainting.} A parallel paradigm tackles NVS through a warping-and-inpainting formulation. This pipeline stems from the success of diffusion-based inpainting models \cite{lugmayr2022repaint, Kim2025rad, zheng2025lanpaint}, which demonstrate a strong ability to plausibly fill semantic blanks in unobserved regions. Methods like GenWarp \cite{seo2024genwarp} and VistaDream \cite{wang2024vistadream} synthesize novel views by first warping the source image using depth priors, followed by utilizing an inpainting network based on image generation to hallucinate the missing pixels caused by occlusion. Extending this to dynamic 4D scenes, recent video-based methods such as ViewCrafter \cite{yu2024viewcrafter} and FlexWorld \cite{chen2025flexworld} lift the source video into a 3D point cloud, render it from the target trajectory, and employ a video generation model to temporally inpaint the resulting visual holes. However, a major bottleneck of these approaches is the absolute requirement to learn or heavily fine-tune a task-specific video inpainting model to handle the severe warping artifacts and domain gaps. To alleviate this heavy training burden, recent works \cite{hou2024training, hu2024motionmaster, ling2024motionclone} like NVS-Solver \cite{you2025nvs} have explored training-free zero-shot novel view synthesis by adaptively modulating the score function of pre-trained models. In our pipeline, we align with this training-free warping and inpainting philosophy but strictly diverge by finding a good initialization rather than controlling the generation mapping relation.

\textbf{Generation Control via Noise Handling.} Controlling the generative process through explicit noise manipulation has emerged as a powerful paradigm, originating from image editing tasks where inverting the denoising trajectory successfully extracts the structural composition of a source image \cite{Tumanyan2023, mokady2022null, parmar2023zeroshot, huberman2024edit, cyclediffusion, hertz2022prompt}. Extending these principles to video generation introduces the critical challenge of temporal coherency \cite{ceylan2023pix2video, tokenflow2023, yang2023rerender,kasten2021layered, couairon2024videdit, chai2023stablevideo}. More recently, explicit manipulation of the initial noise or intermediate latent space has proven highly effective for dictating motion dynamics without retraining. Techniques in this domain range from applying pre-specified translation vectors to inverted noise \cite{Khachatryan2023} to constructing correlated noise distributions \cite{ge2023} and aligning latent features \cite{ni2023, blattmann2023}. Pushing this concept further, noise warping \cite{hiwyn,gowiththeflow} demonstrates that spatially warping the initial noise map according to optical flow inherently establishes robust temporal consistency for zero-shot video editing. While these methods firmly establish that structural layout and motion can be dictated by carefully crafted initial noise, they have no constraints on the unseen regions and largely rely on the quality of the flow for the observed regions, which is quite low for an NVS problem due to the noise of the estimated depth and the incompleteness of the object. Our framework elegantly overcomes these limitations by propagating valid semantic priors to meaningfully constrain and hallucinate the unseen regions and adaptively trusting only the reliable visible regions.

\section{\nickname{}}
\label{sec:method}
\begin{figure}[t]
\setlength{\abovecaptionskip}{ 2 pt}
\setlength{\belowcaptionskip}{ -2 pt}
\centering
\centerline{\includegraphics[width=1.\linewidth]{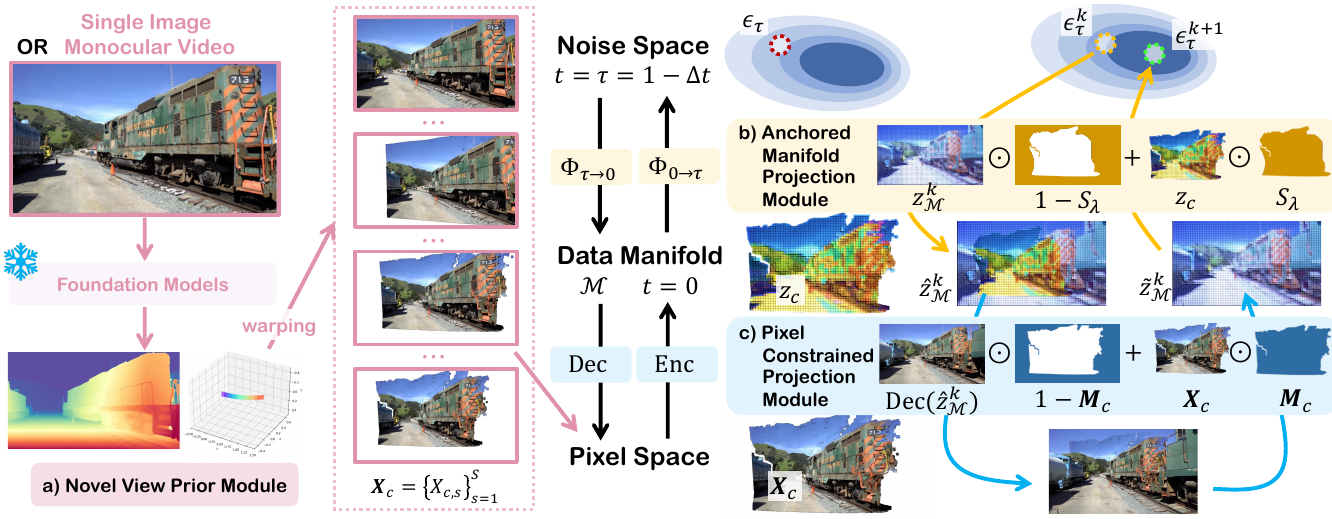}}
    \caption{An illustration of the three modules in our framework.}
    \label{fig:overview}
\end{figure}

We introduce a novel view synthesis (NVS) framework built on pre-trained flow-based video generation models, with a core manifold alternative projection scheme that unifies global semantic consistency from the video data manifold and pixel-level geometric constraints from the 3D warping prior, as shown in Figure \ref{fig:overview}. This section first covers the preliminaries of flow-based video generation (Section \ref{sec:method-preliminary}), then details our pipeline components: novel view prior construction (Section \ref{sec:method-prior}), anchored manifold projection (Section \ref{sec:method-amp}), pixel-constrained projection (Section \ref{sec:method-pcp}), and auxiliary trajectory re-anchoring (Section \ref{sec:method-reanchoring}).

\subsection{Preliminaries on Flow-Based Video Generation}\label{sec:method-preliminary}

Our framework is compatible with deterministic flow-based video generation models, including Flow Matching \cite{lipman2023flowmatching,liu2022flow} and Probability Flow ODE (PF-ODE) \cite{lu2022dpmsolver}. We formalize core components with compact notation below.

A pre-trained video Variational Autoencoder (VAE) \cite{rombach2022LDM} maps an input video $\boldsymbol{X}\in\mathbb{R}^{F\times H\times W\times3}$ to a low-dimensional latent space $\mathcal{Z}\in\mathbb{R}^{F_\mathcal{Z}\times H_\mathcal{Z}\times W_\mathcal{Z}\times d_\mathcal{Z}}$ via the encoder $\mathrm{Enc}(\cdot)$, and reconstructs it via the decoder $\mathrm{Dec}(\cdot)$. All semantically valid natural videos are mapped to a smooth low-dimensional data manifold $\mathcal{M}\in\mathcal{Z}$, which forms the support of the real video distribution $p_0(z)$. Latent codes on $\mathcal{M}$ are decoded to plausible, temporally consistent videos, while out-of-distribution (OOD) codes lie outside $\mathcal{M}$ and produce invalid content.

Flow-based video generation models learn a continuous vector field $v_t(z)$ governing latent code evolution over $t\in[0,1]$ from the standard Gaussian prior noise space $p_1(z)=\mathcal{N}(0,I)$ at $t=1$ to the data space $z\in\mathcal{M}$ at $t=0$, following the Ordinary Differential Equation (ODE):
\begin{equation}
    \frac{dz}{dt}=v(z,t).
\end{equation}

For simplicity, we denote the flow mapping $\Phi_{t\rightarrow s}(z)$ as the ODE solution mapping $z_t$ from time $t$ to $s$, with two core properties: \textbf{1) Forward (noising) flow}: $\Phi_{0\rightarrow t}(z_0)$ maps a data latent $z_o\in\mathcal{M}$ to a noised code $z_t$ as time $t$; \textbf{2) Reverse (denoising) flow}: $\Phi_{t\rightarrow 0}(z_t)$ maps a noised code $z_t$ back to the data manifold at $t=0$. For a well-trained model, the reverse flow maps any Gaussian sample at $t=1$ to a valid video latent. Our framework is agnostic to the vector field formulation, as long as the above flow properties hold.

\subsection{Novel View Prior Module}\label{sec:method-prior}

Our goal is to generate a multi-frame novel view video from a single monocular reference video $\boldsymbol{X}=\{X_f\}_{f=1}^F\in\mathbb{R}^{F\times H\times W\times3}$ (or a single image with $F=1$). We use depth and pose estimation foundation models such as VGGT \cite{wang2025vggt} and VIPE \cite{huang2025vipe} to jointly recover per-frame depth maps $\boldsymbol{D}=\{D_f\}_{f=1}^F$ and the corresponding source camera trajectory $\mathcal{C}_{src}=\{\mathcal{C}_f=(K_f,[\boldsymbol{R}_f\mid\boldsymbol{T}_f]\in\mathrm{SE}(3))\}_{f=1}^F$. Then we can define the given user-specified target camera trajectory in the same world coordinate as $\mathcal{C}=\{c_s=(K_s,[\boldsymbol{R}_s\mid\boldsymbol{T}_s])\}_{s=1}^S$. 

Following existing works \cite{you2025nvs}, we use a warping operator $\mathcal{W}$ to project the source video content, along with the corresponding depth and source poses, to each target view as follows:
\begin{equation}
    X_{c,s},M_s=\mathcal{W}(X_s,D_s,\mathcal{C}_s,c_s),
\end{equation}
and similarly for single-image input. This operation outputs a partial warped target frame $X_{c,s}$, along with a binary visibility mask $M_s$. Stacking the outputs across all target poses yields the full warped video prior $\boldsymbol{X}_c=\{X_{c,s}\}_{s=1}^S$ and its aligned visibility mask $\boldsymbol{M}_c=\{M_s\}_{s=1}^S$.

We map the warped prior $\boldsymbol{X}_c$ to the latent space via pre-trained VAE encoder to get the latent prior $z_c=\mathrm{Enc}(\boldsymbol{X}_c)$. In an ideal scenario where $z_c$ is a clean and valid latent lying strictly on the natural video data manifold $\mathcal{M}$, one could simply obtain an optimal initial noise at $t=1$ via the forward flow $z_1=\Phi_{0\rightarrow 1}(z_c)$. However, our constructed latent prior $z_c$ is inherently flawed: 1) visible regions of $z_c$ contain out-of-distribution artifacts from depth noise, occlusions, and view-dependent appearance changes; 2) unobserved regions are uninformative and unconstrained, \ie, black or empty in invisible pixel space, providing no valid semantic or geometric cues to align generated content with the source scene.

Consequently, directly utilizing this initial noise $z_1$ yields corrupted generation trajectories and is therefore insufficient. Instead, we define a near-noise starting state $z_\tau=\Phi_{0\rightarrow \tau}(z_c)$ at an early timestep $t=\tau=1-\Delta t$ via the forward flow, and then iteratively refine it to find an optimal initial noise that balances both global plausibility and precise view alignment.

\subsection{Anchored Manifold Projection Module}\label{sec:method-amp}

To address this, we propose the \textbf{Anchored Manifold Projection (AMP)} module as the first component of our alternating projection framework. Intuitively, AMP serves two critical functions: 1) it projects plausible geometric and semantic information into the invisible unconstrained regions, thus enabling plausible completion; 2) it searches for a latent state with higher quality on the manifold $\mathcal{M}$ around the flowed state $z_c$, fixing the artifacts in the visible regions while being anchored by the global constraints. We achieve these two goals through the following procedures. 

Given the current initial noise prior $z_\tau$ from Section \ref{sec:method-prior}, we denote it as our starting initial state $z_\tau^0$ and iteratively update it to optimal points $z_\tau^*=z_\tau^K$ by $K$ iterations. For each iteration $k$, we first use reverse flow to project the initial state $z_\tau^k$ toward the data manifold as $z_\mathcal{M}^k$. To bypass the computational bottleneck of full-step ODE integration, we employ a one-step large-stride Euler approximation of the generative reverse flow as follows:
\begin{equation}
    z_\mathcal{M}^k=\Phi_{\tau\rightarrow 0}(z_\tau^k)\approx z_\tau^k-\tau v_\theta(z_\tau^k,\tau),
\end{equation}
where $v_\theta(x,t)$ is the flow field function of the pre-trained video generation model. By passing the state through the pre-trained vector field $v_\theta$, the model naturally leverages its learned spatial-temporal prior to broadcast the information from the visible region to the invisible region, pushing $z_\mathcal{M}^k$ one-step closer to $\mathcal{M}$. 

However, since spatial-temporal attention in vector field $v_\theta$ also broadcasts noisy information from unobserved regions to the observed region, some of which is good and some of which is harmful, this unconstrained forward flow inevitably alters the highly confident visible regions. To strictly enforce the view prior, we need to anchor the estimated state $z_\mathcal{M}^k$ with the given prior constraint $z_c$. Intuitively, if one latent in $z_\mathcal{M}^k$ is similar to one in $z_c$, it means that the generation model $v_\theta$ tends not to fix it, which implies that it is a stable high-quality prior. Thus, we choose to believe this given prior. By contrast, if the latent is very different from $z_c$, it means the generation model tends to fix it significantly, which implies this latent tends to contain artifacts. Therefore, based on this intuition, we perform a latent blend operation to yield a harmonized state:
\begin{equation}
\label{eqn:amp-blending}
    \hat{z}_\mathcal{M}^k=S_\lambda\odot z_c + (1-S_\lambda)\odot z_\mathcal{M}^k,
\end{equation}
where $S_\lambda$ is the feature-level similarity mask with threshold $\lambda$, defined as
\begin{equation}
    S_\lambda=\mathbbm{1}(\mathrm{CosSim}(z_\mathcal{M}^k,z_c)>\lambda)\in\{0,1\}^{F_\mathcal{Z}\times H_\mathcal{Z}\times W_\mathcal{Z}}.
\end{equation}

This dynamic latent blending preserves the model's high-quality hallucinations in artifact regions while forcefully overwriting the visible geometry back to the NVS prior $z_c$. Finally, we map this harmonized latent back to the near-noise timestamp $\tau$ again via the direct forward process:
\begin{equation}
\label{eqn:amp-add-noise}
    z_\tau^{k+1}=\Phi_{0\rightarrow\tau}(\hat{z}_\mathcal{M}^k)=(1-\tau)\cdot\hat{z}_\mathcal{M}^k+\tau\cdot\epsilon,
\end{equation}
where $\epsilon\sim\mathcal{N}(0,I)$ is a random standard Gaussian noise.

\begin{figure}[t]
\setlength{\abovecaptionskip}{ 2 pt}
\setlength{\belowcaptionskip}{ -2 pt}
\centering
\centerline{\includegraphics[width=1.\linewidth]{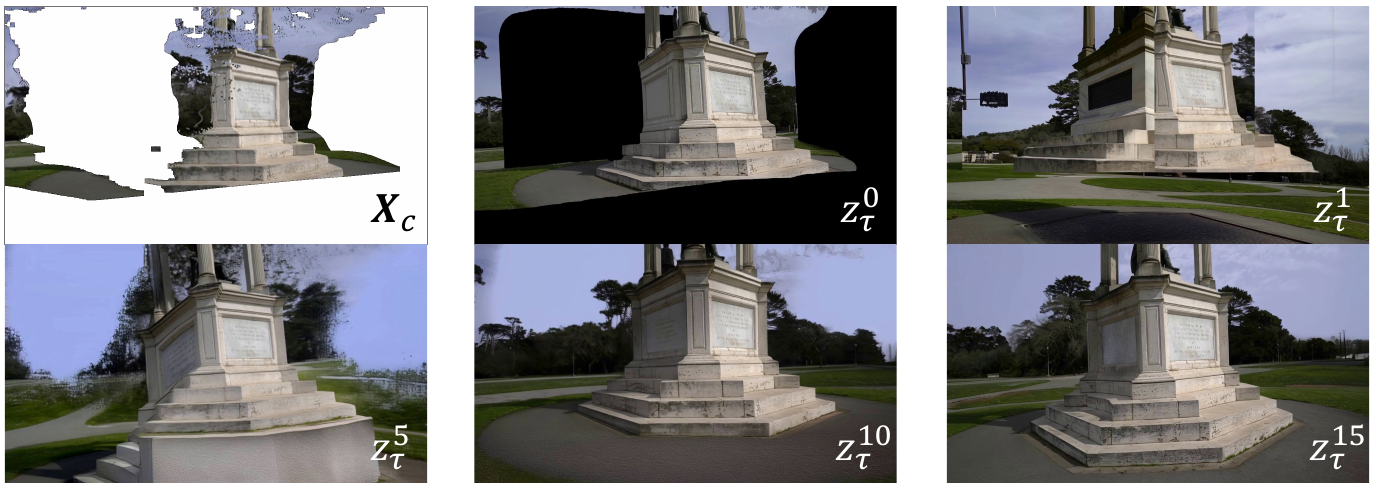}}
    \caption{Qualitative examples show that our alternative projections iteratively refine the noise quality, leading the generated results to higher-fidelity, better view consistency, and more continuous semantics.}
    \label{fig:alternative-qualitative}
\end{figure}

\subsection{Pixel-Constrained Projection Module}\label{sec:method-pcp}

While the AMP module tries to project the state $\hat{z}_\mathcal{M}^k$ toward the data manifold, the similarity mask in the early iterations could be too sparse due to excessively strong noise from unconstrained parts, causing only a small fraction of the latent codes to be successfully anchored. Furthermore, even when latents are successfully anchored, the highly compressed and imperfectly disentangled nature of the VAE latent space $\mathcal{Z}$ means that latent-level blending cannot fully substitute for image-space blending. Latent blending inevitably induces spatial bleeding and global structural shifts, pushing the latent code slightly off the constraints. To solve these problems, we introduce the \textbf{Pixel-Constrained Projection (PCP)} module as the second component of our alternating projection framework. 

After getting $\hat{z}_\mathcal{M}^k$ from Equation \ref{eqn:amp-blending}, instead of directly mapping it back to noise space using Equation \ref{eqn:amp-add-noise}, we decode it to the interpretable pixel space to blend the prior $\boldsymbol{X}_C$ according to our precise binary visibility mask $\boldsymbol{M}_c$ and then encode it back to the latent space via the pre-trained VAE. Mathematically, it is defined as 
\begin{equation}
    \tilde{z}_\mathcal{M}^k = \mathrm{Enc}(\boldsymbol{M}_c\odot\boldsymbol{X}_c+(1-\boldsymbol{M}_c)\odot\mathrm{Dec(\hat{z}_\mathcal{M}^k)}),
\end{equation}
ensuring that every observable pixel from the source view remains geometrically immutable. Finally, we map $\tilde{z}_\mathcal{M}^k$ back to the noise space using Equation \ref{eqn:amp-add-noise}.

\paragraph{Alternating Projection Formulation:} Our proposed pipeline can be theoretically formulated as Alternating Projections onto Sets. A single AMP pass pulls the state toward the natural video manifold $\mathcal{M}$, but its aggressive approximation relaxes spatial constraints; whereas a single PCP pass forcefully projects the state back onto the strict geometric constraint subspace $\mathcal{V}$, but introduces VAE encoding artifacts that push it away from $\mathcal{M}$. Since neither projection independently satisfies both conditions, as shown in Figure \ref{fig:alternative-qualitative}, using the manifold alternating projection iteratively between AMP and PCP as
\begin{equation}
    \ldots\rightarrow z_\tau^k\xrightarrow{\mathrm{AMP_{stage\ 1}}}\hat{z}_\mathcal{M}^k\xrightarrow{\mathrm{PCP}}\tilde{z}_\mathcal{M}^k\xrightarrow{\mathrm{AMP_{stage\ 2}}} z_\tau^{k+1}\rightarrow\ldots
\end{equation}
monotonically leads the optimal latent state towards $\mathcal{M}\cap\mathcal{V}$. The convergence proof is provided in Appendix A.

\subsection{Auxiliary Trajectory Re-anchoring}\label{sec:method-reanchoring}

While the alternating projection pipeline successfully yields an optimized initial noise $z_\tau^*$, the usage of large-stride linear Euler approximation inevitably introduces integration errors. Because the learned vector field $v_\theta$ is highly non-linear, deploying a standard ODE solver initialized directly from $z_\tau^*$ may cause the generation trajectory to gradually drift away from the target geometric constraints. 

Fortunately, extensive literature \cite{chefer2025videojam,flowmo,jang2026frame} on diffusion and flow-based generative models demonstrates that the global structure, semantic layout, and low-frequency content of a generated sample are predominantly determined during the earliest generation steps. Consequently, we only need to explicitly constrain the trajectory during these initial iterations to guarantee global view alignment. As discussed in Section \ref{sec:method-amp}, an AMP module ensuring the intermediate manifold prediction is dynamically anchored to the reliable geometric priors. Therefore, at a new denoise timestamp $t=\tau'<\tau$, we can simply using the AMP module on $z_{\tau'}$ by treating the current intermediate latent $z_{\tau'}$ as a new initialization point to re-anchor the denoising trajectory to our NVS constraints.

\section{Experiments}
\label{sec:exp}

\textbf{Implementation details:} we implement our method based on the recent video generation model Wan2.2-I2V-A14B \cite{wan2025}. The whole denoising steps are set to be 15, and we use time shift 3.5 to control the generation pace. All other configurations are kept as the original setting and its vanilla UNIPC solver \cite{zhao2023unipc} is used to perform denoising. All our hyperparameters are discussed in ablation study in Section \ref{sec:ablation} and more implementation details are in Appendix B.

\textbf{Datasets:} We evaluate our method for static single-image novel view synthesis on two public datasets: 1) \textbf{Tanks-and-Temples dataset} (TNT)\cite{Knapitsch2017}: This dataset contains both indoor and outdoor scenes undergoing large global view changes. Following \cite{chen2025flexworld}, we randomly sampled valid video clips from 14 scenes of its test sets, in which we only pick out 49 valid scenes with reasonable camera trajectories for evaluation, while ignoring those scenes with totally unconstrained views, \ie, there exist frames totally empty after warping the first frame. Notably, since there is no official camera pose given, we use VGGT \cite{wang2025vggt} with calibration by COLMAP \cite{pan2024glomap} to recover camera trajectories. 2) \textbf{LLFF dataset} \cite{mildenhall2019llff}: this dataset contains 8 high-quality scenes, undergoing fast forward-facing camera motions. The high-frequency periodic appearance patterns and high-speed camera motions make it extremely challenging for NVS task. 

We evaluate our method for monocular video novel view synthesis on the 3) \textbf{DAVIS dataset} \cite{davis}: this dataset contains 89 monocular videos, each featuring motions varying from human behaviors to natural environments. We use VIPE \cite{huang2025vipe} to reconstruct the camera trajectory for each scene and compose it with a manually designed relative trajectory (following the protocol of ReCamMaster \cite{bai2025recammaster}) as our target trajectory. More datasets details are provided in Appendix D.

\begin{table}[t!]\tabcolsep=0.05cm  
\centering
\caption{Quantitative results of all methods for visual quality and camera accuracy on Tanks-and-Temples and LLFF datasets.} 
\resizebox{1\linewidth}{!}{
\begin{tabular}{r|cccccc|ccc}
\toprule 
 & \multicolumn{9}{c}{Tanks and Temples} \\
 \cmidrule{2-10}
 & \multicolumn{6}{c|}{Visual Quality} & \multicolumn{3}{c}{Camera Accuracy} \\
 \cmidrule{2-10}
 & PSNR$\uparrow$ & SSIM $\uparrow$ & LPIPS$\downarrow$ & FID-192$\downarrow$ & FID-2048$\downarrow$ & CLIP-S$\uparrow$ & ATE$\downarrow$ & RPE-T$\downarrow$ & RPE-R$\downarrow$ \\
 \midrule
TrajectoryAttention \cite{trajattn} & 13.064 & 0.577 & 0.522 & 14.377 & 151.418 & 0.899 & 0.061 & 0.007 & 0.010\\
ReCamMaster \cite{bai2025recammaster} & 11.549 & 0.533 & 0.594 & 9.542 & 139.487 & 0.903 & 0.099 & 0.004 & 0.006\\
ViewCrafter \cite{yu2024viewcrafter} & 13.675 & 0.590 & 0.564 & 13.748 & 133.219 & 0.902 & 0.010 & 0.004 & \bronze{0.003}\\
FlexWorld \cite{chen2025flexworld} & \bronze{14.655} & \gold{0.601} & 0.531 & \bronze{4.941} & \bronze{92.686} & \bronze{0.918} & \gold{0.005} & \gold{0.002} & \silver{0.002}\\
NVS-Solver \cite{you2025nvs} & 12.356 & 0.547 & 0.548 & 6.376 & 131.179 & 0.892 & 0.079 & 0.006 & 0.012\\
LanPaint \cite{zheng2025lanpaint} & 13.299 & 0.544 & 0.575 & 13.254 & 196.341 & 0.834 & \silver{0.006} & \gold{0.002} & \gold{0.001}\\
Reversed Noise \cite{liu2022flow} & 10.715 & 0.435 & \bronze{0.520} & 17.275 & 160.994 & 0.902 & 0.029 & {0.003} & 0.005\\
Flow Warped Noise \cite{hiwyn} & 11.353 & 0.510 & 0.552 & 6.671 & 130.484 & 0.909 & 0.080 & 0.007 & 0.013\\
\midrule
\textbf{\nickname{} (Wan2.1)} & \silver{14.876} & \bronze{0.591} & \gold{0.481} & \silver{4.735} & \silver{89.527} & \silver{0.922} & 0.010 & \gold{0.002} & 0.004\\
\textbf{\nickname{} (Ours)} & \gold{15.250} & \silver{0.596} & \silver{0.486} & \gold{3.712} & \gold{83.341} & \gold{0.935} & \bronze{0.008} & \gold{0.002} & \bronze{0.003} \\[+0.1em]
\toprule
 & \multicolumn{9}{c}{LLFF} \\
 \cmidrule{2-10}
 & \multicolumn{6}{c|}{Visual Quality} & \multicolumn{3}{c}{Camera Accuracy} \\
 \cmidrule{2-10}
 & PSNR$\uparrow$ & SSIM $\uparrow$ & LPIPS$\downarrow$ & FID-192$\downarrow$ & FID-2048$\downarrow$ & CLIP-S$\uparrow$ & ATE$\downarrow$ & RPE-T$\downarrow$ & RPE-R$\downarrow$ \\
 \midrule
TrajectoryAttention \cite{trajattn} & \bronze{12.557} & \bronze{0.356} & \silver{0.498} & 19.930 & \bronze{123.126} & 0.923 & 1.644 & 0.481 & 0.045\\
ReCamMaster \cite{bai2025recammaster} & 10.928 & 0.312 & 0.561 & 10.994 & 173.737 & 0.885 & 1.300 & 0.500 & 0.042\\
ViewCrafter \cite{yu2024viewcrafter} & 11.573 & 0.344 & 0.594 & 22.091 & 135.944 & 0.907 & 0.864 & \silver{0.470} & \bronze{0.035}\\
FlexWorld \cite{chen2025flexworld} & \silver{13.119} & \gold{0.378} & 0.538 & 12.549 & \silver{97.500} & \bronze{0.930} & \bronze{0.618} & \silver{0.470} & \silver{0.031}\\
NVS-Solver \cite{you2025nvs} & 11.418 & 0.304 & 0.531 & \silver{6.739} & 150.368 & 0.895 & 1.925 & 0.483 & 0.054\\
LanPaint \cite{zheng2025lanpaint} & 11.946 & 0.334 & 0.506 & 13.021 & 196.799 & 0.862 & \gold{0.543} & \gold{0.449} & \gold{0.013}\\
Reversed Noise \cite{liu2022flow} & 10.866 & 0.287 & \bronze{0.505} & \bronze{7.221} & 148.187 & 0.923 & 1.142 & 0.489 & 0.040\\
Flow Warped Noise \cite{hiwyn} & 10.678 & 0.261 & 0.538 & 19.483 & 130.105 & \silver{0.941} & 2.011 & 0.501 & 0.043\\
\midrule
\textbf{\nickname{} (Ours)} & \gold{13.385} & \silver{0.360} & \gold{0.432} & \gold{3.586} & \gold{80.888} & \gold{0.953} & \silver{0.551} & \bronze{0.476} & \silver{0.031}\\[+0.1em] 
\bottomrule
\end{tabular}
}
\label{tab:static}
\vspace{-0.2cm}
\end{table}

\textbf{Baselines:} We compare with 3 groups of baselines: 1) camera-controlled video generation method TrajectoryAttention \cite{trajattn} and ReCamMaster \cite{bai2025recammaster}; 2) methods under warping and inpainting scheme, including task-specific-trained methods ViewCrafter \cite{yu2024viewcrafter} and FlexWorld \cite{chen2025flexworld}, and training-free methods NVS-Solver \cite{you2025nvs}. In order to show the gap between NVS and traditional inpainting problem, we also include powerful training-free inpainting method LanPaint \cite{zheng2025lanpaint}; 3) noise handling methods, including the direct reversed noise via the rectified-flow forward process \cite{liu2022flow} (generating with $z_\tau$ directly) and flow warped noise following Go-With-The-Flow \cite{gowiththeflow}. For all the baselines, we unify the same depth and camera priors and align the scale as required by different models. For training-free methods LanPaint, reversed noise, and flow warped noise, we adapt them to the same generation backbone Wan2.2-I2V-A14B \cite{wan2025} as ours. More implementation details about baselines are shown in Appendix E. 

\textbf{Metrics:} Following \cite{chen2025flexworld, bai2025recammaster, you2025nvs}, we evaluate NVS in both visual quality and camera accuracy. 1) For visual quality, we compare fidelity and coherence with text, quantified with PSNR, SSIM \cite{ssim}, LPIPS \cite{lpips}, Frechet Image Distance \cite{heusel2017fid} (FID), Frechet Video Distance \cite{fvd} (FVD), and CLIP Similarity (CLIP-S), respectively. We note that baselines perform differently with different feature dimensions for FID, so we report both FID-192 and FID-2048. CLIP Similarity refers to the average of per-frame CLIP \cite{clip} similarity between generated frames and ground-truth frames. 2) For camera accuracy, we report Absolute Trajectory Error (ATE) \cite{Goel1999RobustLU} and Relative Pose Error for translation (RPE-T) and rotation (RPE-R) \cite{Goel1999RobustLU}. More details about metrics implementation are in Appendix F.

\begin{figure}[t]
\setlength{\abovecaptionskip}{ 2 pt}
\setlength{\belowcaptionskip}{ -2 pt}
\centering
\centerline{\includegraphics[width=1.\linewidth]{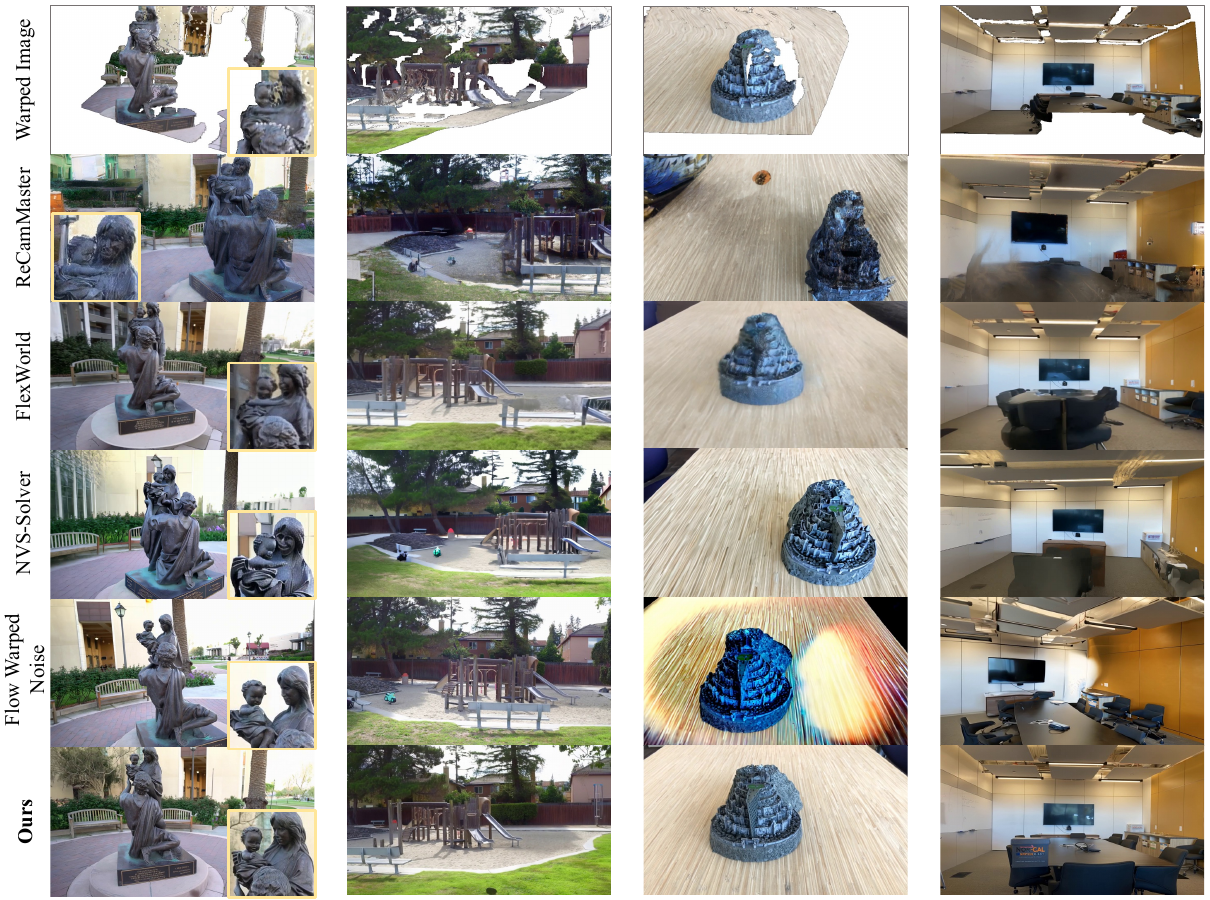}}
    \caption{Qualitative results of top 4 baselines and our methods on Tanks-and-Temples dataset and LLFF dataset.}
    \label{fig:qual-static}\vspace{-0.3cm}
\end{figure}

\vspace{-0.2cm}
\subsection{Static Single-image Novel View Synthesis}

As reported in Table \ref{tab:static}, our proposed \nickname{} framework achieves state-of-the-art performance in both visual quality and semantic consistency across the Tanks-and-Temples and LLFF datasets, securing the best scores in almost all perceptual metrics, including PSNR, LPIPS, FID, and CLIP-S. The quantitative comparisons clearly demonstrate that merely manipulating the initial noise for visible regions, as attempted by baselines such as Reversed Noise and Flow Warped Noise, is fundamentally insufficient to yield reasonable novel view synthesis, which is reflected in their substantial performance drops. Alongside these superior visual results, \nickname{} maintains commendable camera control accuracy that rivals heavily constrained optimization methods. These quantitative gains are strongly corroborated by our qualitative results as shown in Figure \ref{fig:qual-static}: baseline models  stubbornly retain warping noise due to an absolute over-reliance on the flawed geometric prior, whereas our alternating projection approach robustly corrects visible artifacts while seamlessly hallucinating geometrically coherent content in unseen areas.

\begin{table}[t!]\tabcolsep=0.05cm  
\centering
\caption{Quantitative results of all methods for dynamic monocular video novel view synthesis on DAVIS dataset.} 
\resizebox{0.9\linewidth}{!}{
\begin{tabular}{r|ccc|ccc}
\toprule 
 & \multicolumn{3}{c|}{Visual Quality} & \multicolumn{3}{c}{Camera Accuracy} \\
 \cmidrule{2-7}
 & FID-192$\downarrow$ & FID-2048$\downarrow$ & FVD$\downarrow$ & ATE$\downarrow$ & RPE-T$\downarrow$ & RPE-R$\downarrow$ \\
 \midrule
TrajectoryAttention \cite{trajattn} & 15.320 & 115.751 & 496.158 & 1.166 & 0.190 & 0.023\\
ReCamMaster \cite{bai2025recammaster} & 3.772 & 115.584 & 532.946 & \silver{0.709} & \silver{0.078} & \silver{0.012}\\
ViewCrafter \cite{yu2024viewcrafter} & 12.604 & 127.991 & 490.705 & \bronze{0.794} & 0.111 & 0.014\\
FlexWorld \cite{chen2025flexworld} & 2.948 & \silver{77.179} & \silver{243.930} & 0.805 & 0.101 & \bronze{0.013}\\
NVS-Solver \cite{you2025nvs} & \bronze{2.654} & \bronze{81.553} & 341.127 & 0.996 & 0.159 & 0.021\\
LanPaint \cite{zheng2025lanpaint} & \silver{2.436} & 92.367 & \bronze{258.957} & 1.052 & 0.132 & \gold{0.011}\\
Reversed Noise \cite{liu2022flow} & 3.821 & 92.280 & 321.535 & 0.798 & \bronze{0.090} & \bronze{0.013}\\
Flow Warped Noise \cite{hiwyn} & 14.969 & 106.699 & 479.094 & 1.815 & 0.209 & 0.030\\
\midrule
\textbf{\nickname{} (Ours)} & \gold{1.544} & \gold{68.864} & \gold{234.634} & \gold{0.534} & \gold{0.067} & \gold{0.011}\\[+0.1em] 
\bottomrule
\end{tabular}
}
\label{tab:dynamic}
\vspace{-0.5cm}
\end{table}

\vspace{-0.3cm}
\subsection{Dynamic Monocular Video Novel View Synthesis}

As shown in Table \ref{tab:dynamic}, our \nickname{} framework consistently achieves state-of-the-art performance in dynamic monocular video novel view synthesis on the DAVIS dataset, demonstrating optimal results across visual quality and camera control metrics. In complex dynamic scenes, the entanglement of global camera motion and independent object movement frequently pushes baseline models into out-of-distribution (OOD) states, allowing our method to exhibit an even more pronounced advantage here than in relatively static NVS scenarios. Furthermore, the intricate spatial-temporal correlations introduced by these complex motions often cause baseline models to severely misjudge camera trajectories, leading to a noticeable degradation in their camera control accuracy, whereas our method maintains highly robust and precise camera conditioning with significant leads in ATE and RPE metrics. As shown by further qualitative results in Figure \ref{fig:qual-dynamic}, since depth estimation in dynamic scenes is notoriously inferior to that in static environments, frequently introducing severe warping deformations and noise, baseline methods are highly susceptible to pronounced object distortion and semantic deviation of the main subjects, whereas ours effectively mitigates these geometric artifacts and preserves the structural integrity of the moving subjects.

\begin{figure}[t]
\setlength{\abovecaptionskip}{ 2 pt}
\setlength{\belowcaptionskip}{ -2 pt}
\centering
\centerline{\includegraphics[width=1.\linewidth]{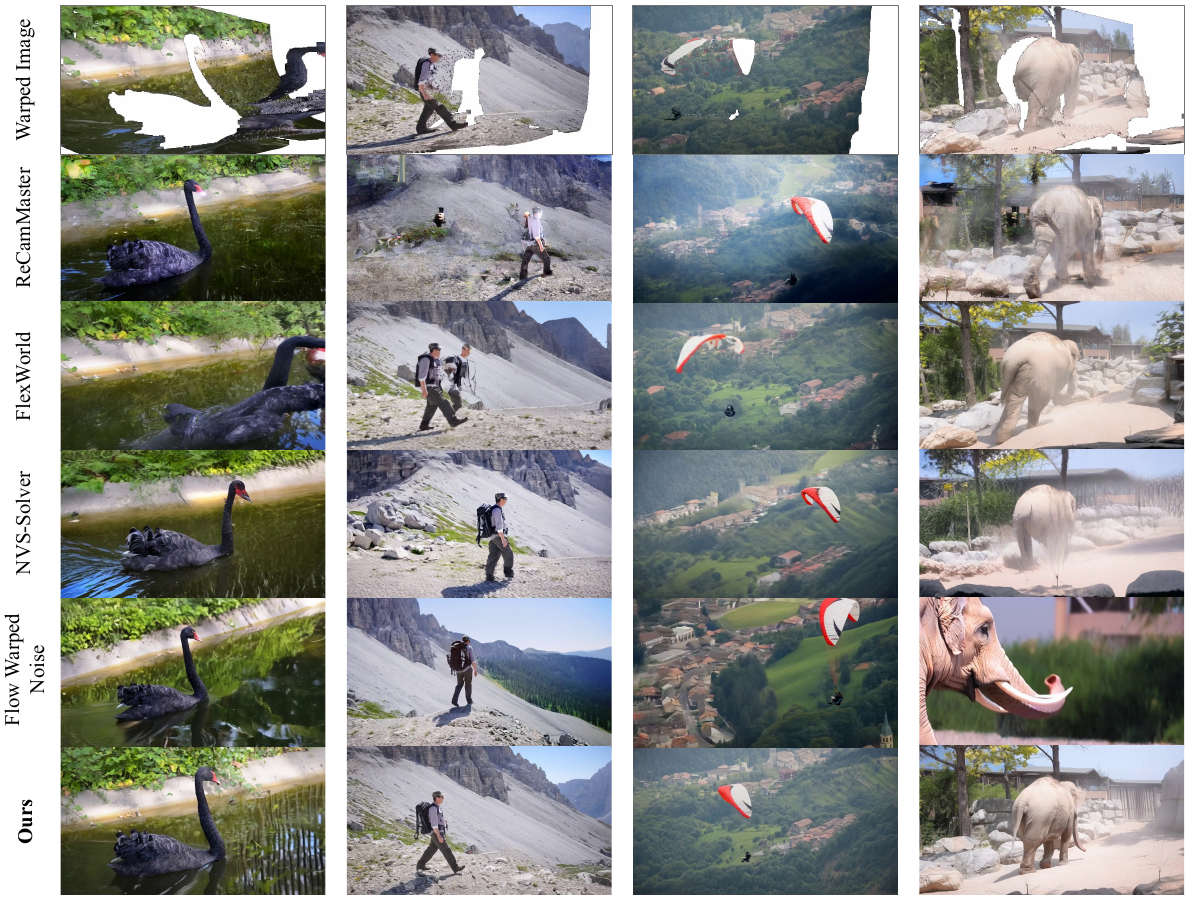}}
    \caption{Qualitative results of top 4 baselines and our methods on DAVIS dataset.}
    \label{fig:qual-dynamic}\vspace{-0.4cm}
\end{figure}

\vspace{-0.3cm}
\subsection{Ablation Study}
\label{sec:ablation}

To evaluate the effectiveness of each component and hyperparameter choices, we conduct the following ablation experiments on the Tanks-and-Temples dataset.

\textbf{(1) Removing all the modules:} We remove all the proposed modules to show the vanilla backbone performances. The method degenerates to generating with $z_\tau$, which is equivalent to a reversed noise baseline. 

\textbf{(2) The number of alternating projection iterations $K$:} we choose $K=15$ in our main experiments, and we ablate the performance with $K=5$ and $K=10$. 

\textbf{(3) The threshold used in the similarity mask in AMP:} we choose $\lambda=0.8$ in our main experiments, and we ablate the performance with $\lambda=0.7$ and $\lambda=0.9$.

\textbf{(4) Removing AMP:} we remove the AMP module.

\textbf{(5) Removing PCP:} we remove the PCP module. 

\textbf{(6) Applying PCP in latent space:} we remove the VAE mappings in the PCP module. Instead, we interpolate the visibility mask $\boldsymbol{M}$ into the size of the latent space as $\boldsymbol{M_\mathcal{Z}}$, and blend the latent state $\hat{z}_\mathcal{M}^k$ with $z_c$ directly as $\tilde{z}_\mathcal{M}^k=\boldsymbol{M_\mathcal{Z}}\odot z_c+(1-\boldsymbol{M_\mathcal{Z}})\odot \hat{z}_\mathcal{M}^k$.

\textbf{(7) The number of denoising steps to apply auxiliary trajectory re-anchoring:} we use the trajectory re-anchoring in the first 3 denoising steps in our main experiments, while we ablate in the first 0, 5, and 8 steps here. 

\textbf{(8) Adding PCP in the trajectory re-anchoring stage:}

As detailed in Table \ref{tab:ablation}, our proposed framework delivers massive performance gains over the naive baseline, with visual metrics steadily converging as the number of alternating projection iterations ($K$) increases. Our constraint mechanisms prove critical for maintaining geometric integrity: enforcing absolute pixel-space boundaries via PCP is irreplaceable, as removing it or applying it solely in the VAE latent space induces spatial bleeding and severe camera decoupling. Within AMP, the similarity threshold ($\lambda$) requires delicate balancing; overly low values admit low-quality warping artifacts, while excessively high values discard too many reliable priors, similarly leading to geometric decoupling and overall view shifts. Finally, the auxiliary trajectory re-anchoring introduces a crucial trade-off between raw visual fidelity and precise camera alignment. While omitting it entirely yields the highest raw image quality, it suffers from noticeable trajectory drift; applying an optimal 3 steps effectively calibrates the camera alignment with negligible visual loss, whereas excessive re-anchoring aggressively harms image quality without yielding further trajectory benefits.

\begin{table}[t!]\tabcolsep=0.05cm  
\centering
\caption{Quantitative results of all ablation studies on Tanks-and-Temples datasets.} 
\resizebox{1\linewidth}{!}{
\begin{tabular}{cccccc|ccc|ccc}
\toprule 
\multicolumn{6}{c|}{ }& \multicolumn{3}{c|}{Visual Quality} & \multicolumn{3}{c}{Camera Accuracy} \\[+0.1em] 
\cmidrule{7-12}
& $K$ & $\lambda$ & AMP & PCP & Re-anchor & FID-192$\downarrow$ & FID-2048$\downarrow$ & CLIP-S$\uparrow$ & ATE$\downarrow$ & RPE-T$\downarrow$ & RPE-R$\downarrow$ \\
 \midrule
1) & - & - & \N & \N & - & 17.275 & 160.994 & 0.902 & 0.029 & 0.003 & 0.005\\
2) & 5 & 0.8 & \Y & \Y & 3 & 4.864 & 96.359 & 0.923 & 0.010 & 0.002 & 0.003\\
2) & 10 & 0.8 & \Y & \Y & 3 & 4.774 & 88.310 & 0.930 & 0.008 & 0.002 & 0.003\\
3) & 15 & 0.7 & \Y & \Y & 3 & 5.240 & 92.013 & 0.926 & 0.010 & 0.002 & 0.003\\
3) & 15 & 0.9 & \Y & \Y & 3 & \bronze{3.976} & 88.237 & \bronze{0.933} & 0.022 & 0.002 & 0.004\\
4) & 15 & 0.8 & \N & \Y & 3 & 4.943 & 91.215 & 0.929 & 0.012 & 0.002 & 0.003 \\
5) & 15 & 0.8 & \Y & \N & 3 & 5.839 & 115.214 & 0.925 & 0.089 & 0.005 & 0.011\\
6) & 15 & 0.8 & \Y & Latent & 3 & 5.661 & 93.343 & 0.928 & 0.011 & 0.002 & 0.003\\
\midrule
7) & 15 & 0.8 & \Y & \Y & 0 & \gold{2.773} & \silver{83.750} & \gold{0.940} & 0.017 & 0.002 & 0.005\\
7) & 15 & 0.8 & \Y & \Y & 5 & 4.674 & \bronze{86.026} & 0.932 & 0.008 & 0.002 & 0.003\\
7) & 15 & 0.8 & \Y & \Y & 8 & 6.326 & 96.721 & 0.926 & {0.007} & 0.002 & 0.002\\
8) & 15 & 0.8 & \Y & \Y & PCP & 6.893 & 107.733 & 0.909 & {0.006} & 0.002 & 0.002\\
\midrule
\textbf{Full} & 15 & 0.8 & \Y & \Y & 3 & \silver{3.712} & \gold{83.341} & \silver{0.935} & 0.008 & 0.002 & 0.003\\[+0.1em] 

\bottomrule
\end{tabular}
}
\label{tab:ablation}
\vspace{-0.3cm}
\end{table}

\subsection{Further Application}

In addition to the above NVS application, we further show the ability of our methods on video-to-video (V2V) future extrapolation. Given a sequence of video frames, we directly mark them as visible part in our framework, while frames in future unobserved time as invisible part. Based on these given video priors, we directly run our method on it, which unifies training-free V2V future extrapolation task and NVS task. We show qualitative results in some selected scenes from DAVIS \cite{davis} and the training set of PhysInOne \cite{zhou2026physinone} as shown in Figure \ref{fig:future-extrap}.

\begin{figure}[t]
\centering
\centerline{\includegraphics[width=1.\linewidth]{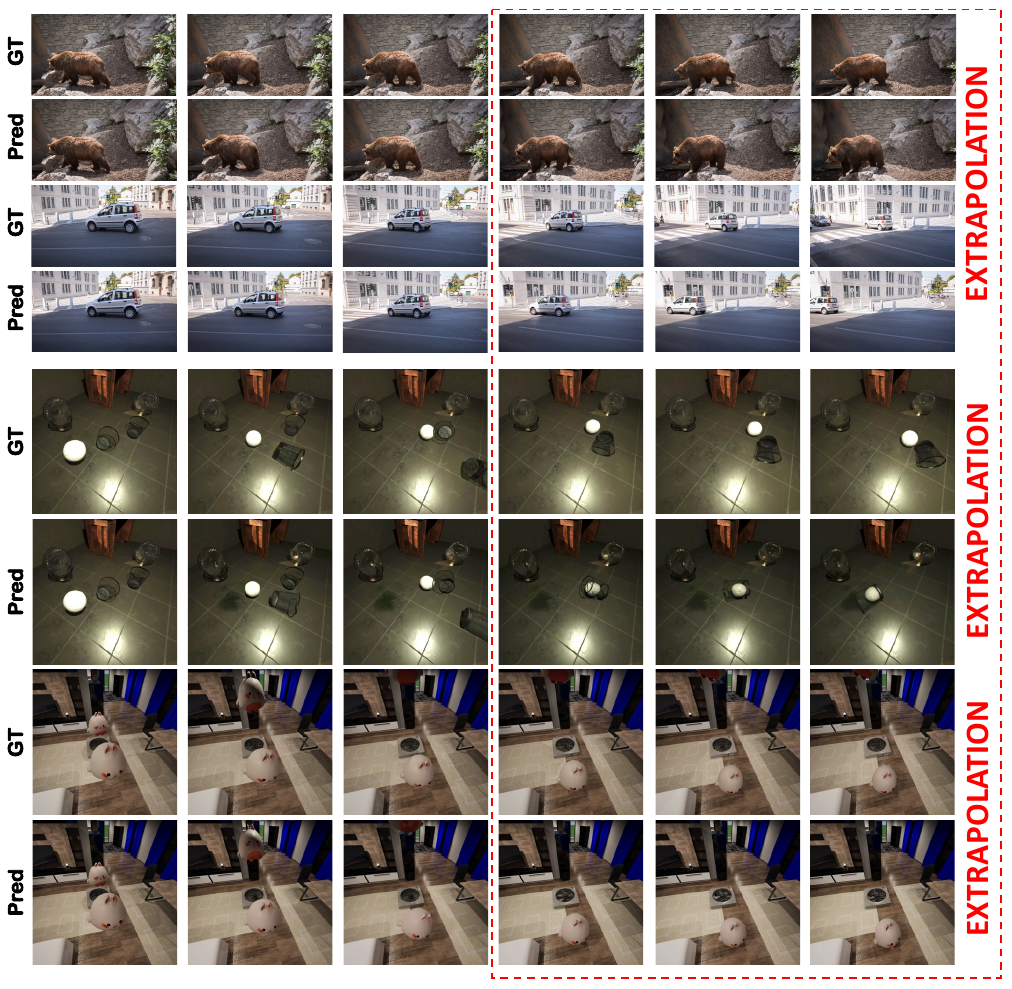}}
    \caption{Qualitative results of video-to-video future extrapolation on DAVIS dataset (upper two scenes) and PhysInOne dataset (bottom two scenes). Our method can provide smooth motion prediction based on the given videos.}
    \label{fig:future-extrap}\vspace{-0.4cm}
\end{figure}

\section{Conclusion}
\label{sec:conclusion}
We present a novel entirely training-free framework that addresses the inherent challenges of severe occlusions and depth ambiguity in monocular novel view synthesis (NVS). Rather than relying on expensive task-specific fine-tuning or heuristic generation guidance, we reframed the NVS problem as discovering the optimal initial noise latent within the native output data manifold of a pre-trained video generation model. Our core contribution lies in the manifold alternating projection scheme: iterating between Anchored Manifold Projection (AMP) for plausible semantic completion and Pixel-Constrained Projection (PCP) for exact geometric boundary enforcement. This mathematically guaranties convergence to a solution that perfectly intersects global semantic realism with precise 3D view alignment. Extensive experiments demonstrate that our method significantly advances state-of-the-art visual quality and geometric consistency for both static image and dynamic video inputs, establishing a robust new paradigm for training-free novel view synthesis. \\

\textbf{Acknowledgments}:
This work was supported in part by Research Grants Council of Hong Kong under Grants 15225522 \& 15219125 \&15228626, and in part by National Natural Science Foundation of China under Grant 62271431.



%
%
\bibliographystyle{splncs04}
\bibliography{references}


\appendix
\section{Proof of the Convergence}\label{app:proof_converge}

In this appendix, we provide a rigorous mathematical proof for the convergence of our proposed \nickname{} framework. We first formalize the core notations and regularity assumptions aligned with the main paper, then we prove the key theorem: the alternating projection iterations converge to the intersection of the natural video data manifold and the NVS geometric constraint set.

\subsection{Notations and Preliminaries}
We first restate and extend the core notations from the main paper for completeness, followed by standard regularity assumptions that hold for well-trained flow-based video generation models.

\textbf{Core Notations:}
\begin{itemize}
    \item $\mathcal{Z}$: The latent space of the pre-trained video VAE, with inner-product-induced norm $\|\cdot\|_2$.
    \item $\mathcal{M}\subseteq\mathcal{Z}$: The smooth, closed, and connected natural video data manifold, which is the support of the real video distribution learned by the pre-trained model. All valid latent codes on $\mathcal{M}$ can be decoded to photorealistic, temporally consistent videos.
    \item $\mathcal{V\subseteq\mathcal{Z}}$: The closed convex set of NVS geometric constraints, defined as
    \begin{equation}
        \mathcal{V}=\{z\in\mathcal{Z}\mid\boldsymbol{M}_c\odot\mathrm{Dec}(z)=\boldsymbol{M}_c\odot\boldsymbol{X}_c\},
    \end{equation}
    where $\boldsymbol{X}_c$ is the warped novel view prior, $\boldsymbol{M}_c$ is the binary visibility mask, and $\odot$ denotes element-wise multiplication.
    \item ODE and flow mappings: The reverse flow ODE of the pre-trained video generation model is
    \begin{equation}
        \frac{dz(t)}{dt}=v_\theta(z(t),t),\quad t\in[0,1],
    \end{equation}
    where $v_\theta(z,t)$ is the learned vector field, $t=1$ corresponds to the standard Gaussian noise space, and $t=0$ corresponds to the data manifold $\mathcal{M}$. We denote:
    \begin{itemize}
        \item $\Phi_{t\rightarrow s}(z)$: the exact ODE solution mapping a latent code $z_t$ at time $t$ to $z_s$ as time s.
        \item $\hat{\Phi}_{\tau\rightarrow 0}(z_\tau)=z_\tau - \tau v_\theta(z_\tau,\tau)$: the one-step Euler approximation of the reverse flow $\Phi_{\tau \rightarrow 0}$ used in our AMP module, where $\tau=1-\Delta t$ is the near-noise timestamp.
    \end{itemize}
    \item Projection operators:
    \begin{itemize}
        \item $P_\mathcal{V}: \mathcal{Z}\rightarrow \mathcal{V}$: the orthogonal projection onto the constraint set $\mathcal{V}$, which exactly corresponds to our PCP module:
        \begin{equation}
            P_\mathcal{V}(z)=\mathrm{Enc}(\boldsymbol{M}_c\odot \boldsymbol{X}_c+(1-\boldsymbol{M}_c)\odot \mathrm{Dec}(z)).
        \end{equation}
        \item $P_\mathcal{M}: \mathcal{Z}\rightarrow \mathcal{M}$: the projection onto the data manifold $\mathcal{M}$, implemented by our AMP module with dynamic feature anchoring.
    \end{itemize}
    \item Iteration sequence: the full alternating projection iteration in the main paper is
    \begin{equation}
        z_\tau^k\xrightarrow{\mathrm{AMP_{stage\ 1}}}\hat{z}_\mathcal{M}^k\xrightarrow{\mathrm{PCP}}\tilde{z}_\mathcal{M}^k\xrightarrow{\mathrm{AMP_{stage\ 2}}} z_\tau^{k+1},
    \end{equation}
    which can be compactly written as the composite mapping:
    \begin{equation}z_\tau^{k+1}=\Phi_{0\rightarrow\tau}\circ P_\mathcal{V}\circ P_\mathcal{M}(\Phi_{\tau\rightarrow0}(z_\tau^k))=(1-\tau)\cdot T(z_\tau^k)+\tau \epsilon_k,
    \end{equation}
    where $T=P_\mathcal{V}\circ P_\mathcal{M}$ is the composite projection operator, and $\epsilon_k\sim \mathcal{N}(0,I)$ is independent and identically distributed (i.d.d) standard Gaussian noise independent of all previous iterations. 
    \item $z_\tau^*$: the unique fixed point of the ideal noise-free iteration, satisfying $z_\tau^*=(1-\tau)\cdot T(\Phi_{\tau\rightarrow0}(z_\tau^*))$.
    \item $e_k=z_\tau^k-z_\tau^*$: the error between the $k$-th iterative value and the optimal fixed point. 
\end{itemize}

\textbf{Regularity Assumptions:} we introduce the following standard and mild regularity assumptions, which are universally satisfied by well-trained large-scale video generation models:
\begin{assumption}[Lipschitz Continuity of Vector Field]
    The learned vector field $v_\theta(z,t)$ is $L$-Lipschitz continuous with respect to $z$ uniformly over $t\in[0,1]$, \ie, for any $z_1,z_2\in \mathcal{Z}$ and $t\in[0,1]$,
    \begin{equation*}
        \|v_\theta(z_1,t)-v_\theta(z_2,t)\|_2\leq L\cdot\|z_1-z_2\|_2,
    \end{equation*}
    where $L > 0$ is the Lipschitz constant.
\end{assumption}

\begin{assumption}[Bounded Second-Order Derivative of ODE Solution]
    The second-order derivative of the ODE solution is uniformly bounded by a constant $M>0$, \ie, for any $z\in\mathcal{Z}$ and $t\in[0,1]$,
    \begin{equation*}
        \|\ddot{z}(t)\|_2=\|\frac{\partial v_\theta}{\partial t}(z,t)+\nabla_zv_\theta(z,t)\cdot v_\theta(z,t)\|_2\leq M.
    \end{equation*}
\end{assumption}

\begin{assumption}[Contractivity of Composite Projection]
    The composite projection operator $T=P_\mathcal{V}\circ P_\mathcal{M}$ is $\gamma$-Lipschitz continuous with $0<\gamma<1$. This holds because: (1) the convex projection $P_\mathcal{V}$ is $1$-Lipschitz non-expansive; (2) the manifold projection $P_\mathcal{V}$ is locally $1$-Lipschitz; (3) the alternating projection between two non-parallel closed sets forms a strict contraction in the constraint neighborhood.
\end{assumption}

\subsection{Convergence to the Intersection $\mathcal{M}\cap\mathcal{V}$}

We now prove that the latent state sequence generated by our alternating projection framework converges to the intersection of the data manifold and the geometric constraint set, which guarantees that the final solution simultaneously satisfies global semantic realism and precise 3D view alignment.

\begin{theorem}[Convergence to $\mathcal{M}\cap\mathcal{V}$]
    Assume $\mathcal{M}\cap\mathcal{V}\neq \emptyset$. For any initial latent code $z_0\in\mathcal{Z}$, the alternating projection sequence $z^{k+1}=P_\mathcal{V}\circ P_\mathcal{M}(z^k)$ converges strongly to a fixed point $z_\mathcal{M}^*\in \mathcal{M}\cap \mathcal{V}$, \ie,
    \begin{equation*}
        \lim_{k\rightarrow\infty}z^k=z_\mathcal{M}^*\in \mathcal{M}\cap \mathcal{V}.
    \end{equation*}
\end{theorem}

\begin{proof}
We complete the proof in three steps:

\paragraph{Step 1: Existence of Fixed Points.}
For any $z^*\in\mathcal{M}\cap\mathcal{V}$, by the definition of projection operators:
\begin{itemize}
    \item $P_\mathcal{V}(z^*)=z^*$, since $z^*$ already lies in the constraint set $\mathcal{V}$;
    \item $P_\mathcal{M}(z^*)=z^*$, since $z^*$ already lies on the data manifold $\mathcal{M}$.
\end{itemize}
Thus, $T(z^*)=P_\mathcal{V}\circ P_\mathcal{M}(z^*)$, meaning all points in $\mathcal{M}\cap\mathcal{V}$ are fixed points of the composite operator $T$. The non-emptiness of $\mathcal{M}\cap\mathcal{V}$ (guaranteed by the existence of valid NVS solutions in the pre-trained model's output space) ensures the existence of fixed points.

\paragraph{Step 2: Monotonic Decrease of Distance to $\mathcal{M}\cap\mathcal{V}$.}
Define the distance function to the intersection set:
\begin{equation}
    d(z)=\min_{z^*\in\mathcal{M}\cap\mathcal{V}}\|z-z^*\|_2.
\end{equation}
For any iteration $z^k$, let $z^*\in\mathcal{M}\cap\mathcal{V}$ be the closest point to $z^k$, \ie, $d(z^k)=\|z^k-z^*\|_2$. By the non-expansive property of projection operators:
\begin{equation}
    \|z^{k+1}-z^*\|_2=\|T(z^k)-T(z^*)\|_2\leq\|P_\mathcal{M}(z^k)-P_\mathcal{M}(z^*)\|_2\leq\|z^k-z^*\|_2=d(z^k).
\end{equation}
Thus, $d(z^{k+1})=\|z^{k+1}-z^*\|_2\leq d(z^k)$, meaning the distance sequence $\{d(z^k)\}$ is monotonically non-increasing and bounded below by $0$. By the Monotone Convergence Theorem, $\{d(z^k)\}$ converges to a non-negative limit $d^*\geq 0$.

\paragraph{Step 3: Proof of $d^*=0$ (Contradiction).} Assume for contradiction that $d^*>0$. Then all accumulation points of the sequence $\{z^k\}$ satisfy $d(z)=d^*$, \ie, they lie on the sphere of radius $d^*$ centered at $\mathcal{M}\cap\mathcal{V}$. Let $z^\infty$ be an accumulation point of $\{z^k\}$, then $T(z^\infty)$ is also an accumulation point with $d(T(z^\infty))=d(z^\infty)=d^*$. if $z^\infty\notin \mathcal{M}\cap\mathcal{V}$, we have two cases:
\begin{itemize}
    \item 1) if $z^\infty\notin\mathcal{M}$: then $P_\mathcal{M}(z^\infty)\neq z^\infty$. By the strict non-expansive property of convex projection onto closed convex sets, $\|P_\mathcal{M}(z^\infty)-z^*\|_2<\|z^\infty-z^*\|_2=d^*$, which implies $d(T(z^\infty))<d^*$, contradicting the assumption $d(T(z^\infty))=d^*$.
    \item 2) if $z^\infty \in\mathcal{M}$ but $z^\infty\notin\mathcal{V}$: then $P_\mathcal{V}(z^\infty)\neq z^\infty$. By the strict non-expansive property of manifold projection, $\|P_\mathcal{M}\circ P_\mathcal{V}(z^\infty)-z^*\|_2<d^*$, which also contradicts the assumption $d(T(z^\infty))=d^*$.
\end{itemize}
Thus, $z^\infty\in \mathcal{M}\cap\mathcal{V}$ and $d^*=0$. The sequence $\{z^k\}$ converges strongly to a fixed point $z_\mathcal{M}^*\in \mathcal{M}\cap\mathcal{V}$.
\end{proof}

\section{More implementation details}\label{app:implementation}

We implement our method based on the recent video generation model Wan2.2-I2V-A14B \cite{wan2025}. The original Wan2.2-I2V-A14B denoises a video for 50 steps with time shift 5 via UNIPC solver. Since we have our alternating projection warm-up, we reduce the generation steps to 15 with time shift 3.5. We choose $\tau$ as the first timestamp in the denoising trajectory, which is 
\begin{equation}
    \tau=\frac{3.5\times\frac{15-1}{15}}{1+(3.5-1)\times\frac{15-1}{15}}=0.98.
\end{equation}
All other hyperparameters are kept the same as the original model. Since our model in total runs 30 forward pass (\ie 15 alternating projection steps and 15 generation steps) of the flow transformer while the original model runs 50 times, the evaluation performance of our model is only $0.6$ times of the original model with the same memory usage.

\section{Running Time \& Computational Overhead} 
We report theoretical computational overhead and practical running time for both our method and baselines in Table \ref{tab:time}. Here $T$ represents the iterations required by the reverse process, and $K$ is the sub-iters required by different methods. We choose the smallest $T$ requried by each models. We admit our method requires higher running time than pre-trained model and we will include this in the limitation section, but our method shows better efficiency compared with other training-free methods like LanPaint and NVS-Solver. Notably, since we are optimizing initial noise, we don't rely on a particular sampler. Therefore, if more advanced faster sampler is used, our method can be faster. However, other denoising controlling methods all depend on one well-designed sampler, thus their $T$ is fixed. 

\begin{table}[h!]\tabcolsep=0.05cm  \vspace{-0.7cm}
\centering
\caption{Running time for all methods.} 
\resizebox{1\linewidth}{!}{
\begin{tabular}{r|c|c|c|c|c|c|c|c|c}
\toprule 
 & TrajAttn & ReCamMaster  & ViewCrafter & FlexWorld & NVS-Solver & LanPaint & Reversed & Flow Warped & Ours\\
 \midrule
Iters & $T$ & $T$ & $T$ & $T$ & $T\times K$ & $T\times K$ & $T$ & $T$ & $T+K$\\
Memory & $1\times$ & $1\times$ & $1\times$ & $1\times$ & $4\times$ & $1\times$ & $1\times$ & $1\times$ & $1\times$\\
Time (min) & 2.5 & 2.1 & 1.9 & 2.3 & 8.8 & 4.1 & 3.5 & 3.7 & 4.0 \\
Sampler & Fixed & Fixed & Any & Any & Fixed & Fixed & Any & Any & Any\\
\bottomrule
\end{tabular}
}
\label{tab:time}
\vspace{-0.4cm}
\end{table}

\section{More details about the datasets.}\label{app:dataset}

\textbf{Tanks-and-Temples dataset \cite{Knapitsch2017}:} There are in-total 14 scenes in this dataset, each containing a long video capturing the $360^\circ$ view of the given scenes. Following FlexWorld\cite{chen2025flexworld}, we clip the scenes into 49-frame-long video clips with a stride of 4. However, we observe that there exist plenty of videos with unconstrained camera trajectories for an NVS problem, \ie there are frames full of empty pixels after re-projecting the depth points from input view to that view, which means there is no constrain for a model to generate in that view and the score is totally meaningless. Therefore, we only select the video clips without unreasonable frames as our test set, and all the related scenes include \textit{Ballroom}, \textit{Family}, \textit{Francis}, \textit{Horse}, \textit{Lighthouse}, \textit{M60}, \textit{Panther}, \textit{Playground}, and \textit{Train}. Since there is no official camera pose given, we use VGGT \cite{wang2025vggt} followed with COLMAP \cite{pan2024glomap} calibration to get the ground-truth cameras and the depth maps for each frame in the video clips. In our experiments, the depth of the first frame is projected with the estimated camera intrinsics from VGGT, and rendered to the estimated ground-truth camera to get warped images as our NVS prior. For those baseline models taking camera trajectory as the input, we directly change the conversion of the camera poses as required and feed them into the models directly. 

\textbf{LLFF dataset \cite{mildenhall2019llff}:} This dataset contains 8 scenes, each containing 20 to 62 frames. In order to fit the input frame numbers of the baselines, we either clip or pad the frame numbers to 49. The challenges of this dataset lie in the large camera motion between two consecutive frames, and the scenes contain many high-frequency visual pattens such as densely-distributed leaves and bones. This dataset provides images along with camera poses and 3D points. In order to obtain a valid depth map for warping, we train a 3D Gaussian \cite{kerbl3Dgaussians} initialized by the given 3D points for each scene to render the depth map.

\textbf{DAVIS dataset \cite{davis}:} There are 90 dynamic scenes, and 89 of them are used in our experiments. In order to fit the input frame numbers of the baselines, we either clip or pad the frame numbers to 49. We abandoned \textit{lucia}, because PnP totally failed when doing the camera calibration. We use VIPE \cite{huang2025vipe} to extract the calibrated dynamic cameras for each video as the input cameras. As for target cameras, we first design 10 relative camera trajectories following ReCamMaster \cite{bai2025recammaster}, as illustrated in Table \ref{tab:DAVIS_cam_traj}, and then compose them with the input camera trajectory. In this way, all the cameras can be ensured to rotate near the input camera, making the targeted moving objects focused in the frame. When designing the relative camera trajectories, we first take the median-depth point along the optical axis of the camera of the first frame as our pivot, then determine the moving stride and center based on this pivot. Notably, since we compose the designed trajectory with the input trajectory, the final target trajectory contains both the complicated information of the original complex trajectories and the various novel views from manually designed trajectories, making the final absolute camera challenging but reasonable. 

For the two static datasets, we select the first frame as the input frame, and warp it to the target view to get our NVS prior. As for the dynamic dataset, we have 49 frames with depth and 49 target cameras for each scene. Thus we warp the $k$-th depth frame into the $k$-the target camera, and finally get a 49-frame target warped video as our NVS prior.

\begin{table}[t!]
\caption{Camera trajectories and corresponding descriptions for DAVIS dataset.}
\label{tab:DAVIS_cam_traj}
\resizebox{\linewidth}{!}
{

\begin{tabular}{c|c|c}
\toprule

Index & Trajectory & Description \\
\midrule
1 & Pan right & Move camera to the right \\
2& Pan left & Move camera to the left \\
3& Tilt up & Move camera upwards \\
4& Tilt down & Move camera downwards \\
5& Zoom in & Move camera forward \\
6& Zoom out & Move camera backward \\
7& Translate up with rotation & Move the camera upwards while rotating it downwards by 25 degrees \\
8& Translate down with rotation & Move the camera downwards while rotating it upwards by 25 degrees \\
9& Translate right with rotation & Move the camera to the right while rotating it to the left by 25 degrees \\
10& Translate left with rotation & Move the camera to the left while rotating it to the right by 25 degrees \\

\bottomrule

\end{tabular}
}
\end{table}

\textbf{Trajectory Distribution.} We show the trajectory distribution for each dataset in Figure \ref{fig:data}, including the maximum rotation angle, maximum translation (normalized by median depth), visible-overlap ratios, and average camera speed (angle/frame) for each scene, where TnT mainly includes larger rotation and smaller visible overlap ratio, LLFF shows higher camera speed and DAVIS contains larger camera translation. Note the absolute camera of DAVIS should be composed with the complex video camera, which is not shown in the figure. We also show a qualitative comparison on almost-invisible results in Figure \ref{fig:data}. 

\begin{figure}[h!]
  \centering
  \includegraphics[width=1.0\linewidth]{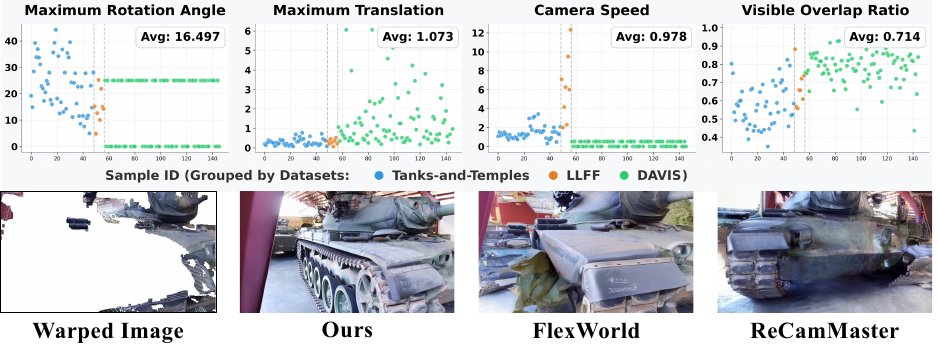}
  \vspace{-0.8cm}
   \caption{Data distribution and almost-invisible results.}
   \vspace{-0.5cm}
   \label{fig:data}
\end{figure}

\section{More details about the baselines.}\label{app:baseline}

We list the adaptations of the baselines below:
\begin{itemize}
    \item For camera-controlled video generation baselines, \ie TrajectoryAttention \cite{trajattn} and ReCamMaster \cite{bai2025recammaster}, we keep everything as the official implementations instead of changing the length of input views into 49. 
    \item For warping-and-inpainting NVS baselines, \ie NVS-Solver \cite{you2025nvs}, ViewCrafter \cite{yu2024viewcrafter}, and FlexWorld \cite{chen2025flexworld}, we keep their core generation models untouched, but adapt their priors the same as ours. 
    \item For LanPaint \cite{zheng2025lanpaint}, we adapt its official Comfy-UI Wan2.2-T2V-A14B video inpainting implementation into a Wan2.2-I2V-A14B version, with the same warped video as the source video and the visibility mask as the inpainting mask. All other hyperparameters are kept the same as its official setting. 
    \item For Reversed Noise, we use the rectified flow forward process \cite{liu2022flow} to add the noise to the warped video as the starting point of the model. All other parameters for generation is kept the same as ours. 
    \item For Flow Warped Noise, we calculate per-pixel optical flow based on the estimated depth and the camera poses. After we sample a random noise for the first frame, we use noise warping following Go-With-The-Flow \cite{gowiththeflow} to warp the first-frame noise into later frames. Then we downsample this 49 noise frames into 13 frames to fit the model requirement by mean-pooling, also following Go-With-The-Flow \cite{gowiththeflow}. 
\end{itemize}

\section{More details about the metrics.}\label{app:metric}

For PSNR, SSIM, and LPIPS, we follow the implementation in FlexWorld \cite{chen2025flexworld}. For FID, we report FID values with both 192 and 2048 feature dimensions extracted from model InceptionV3, using the public torchmetrics package in 1.8.2 version. The CLIP-S value is calculated as the mean of cosine similarities of each frame features extracted from a CLIP-ViT-B/16 model \cite{clip}. When calculating FVD values, the videos are resized to a resolution of 224*224, and clipped to short videos with 16 frames. The final FVD values are calculated as the frechet distance between features of all output and ground-truth video clips, and the features are extracted by StyleGAN-V.

\section{Limitations}\label{app:limitation}

While \nickname{} demonstrates strong zero-shot capabilities for novel view synthesis by optimizing initial noise, it still possesses a few limitations that point to future research directions:  

\textbf{Dependence on the Pre-trained Data Manifold:} First, the efficacy of our framework inherently relies on the fundamental assumption that the target novel view video exists within the output data manifold of the pre-trained video generation model. Although modern foundational video models are trained on massive and diverse real-world datasets, their learned manifold is not entirely exhaustive. In extreme scenarios, such as drastic camera pose transformations with severe occlusions or scenes containing entirely unseen out-of-distribution object categories, the ideal target video may simply fall outside the boundaries of the learned data manifold. Under such circumstances, our manifold alternating projection mechanism cannot physically locate a valid optimal solution, which inevitably leads to a degradation in generation fidelity and structural collapse.  

\textbf{Sensitivity to the Quality of the Geometric Prior:} Second, similar as all other NVS methods, our pipeline imposes certain requirements on the quality of the initial novel view prior. In our design, the Anchored Manifold Projection (AMP) module depends on the pre-trained model's inherent spatial-temporal attention to broadcast valid semantic and geometric information from the observed warped regions to the unconstrained invisible regions. However, if the given prior is severely corrupted (\textit{e.g.}, due to catastrophic depth estimation failures) or if the generation model fails to extract meaningful correlations from the input conditioning, the prior loses its ability to effectively constrain the generation trajectory. Consequently, the denoising process might decouple from the input context, resulting in uncontrolled hallucinations and compromised view consistency.

\section{More qualitative results.}\label{app:more_qualitative}

We provide more qualitative results of all the baselines and our model on all three datasets in Figure \ref{fig:all-tnt} \& \ref{fig:all-llff} \& \ref{fig:all-davis}.

\begin{figure}[t]
\setlength{\abovecaptionskip}{ 2 pt}
\setlength{\belowcaptionskip}{ -2 pt}
\centering
\centerline{\includegraphics[width=1.\linewidth]{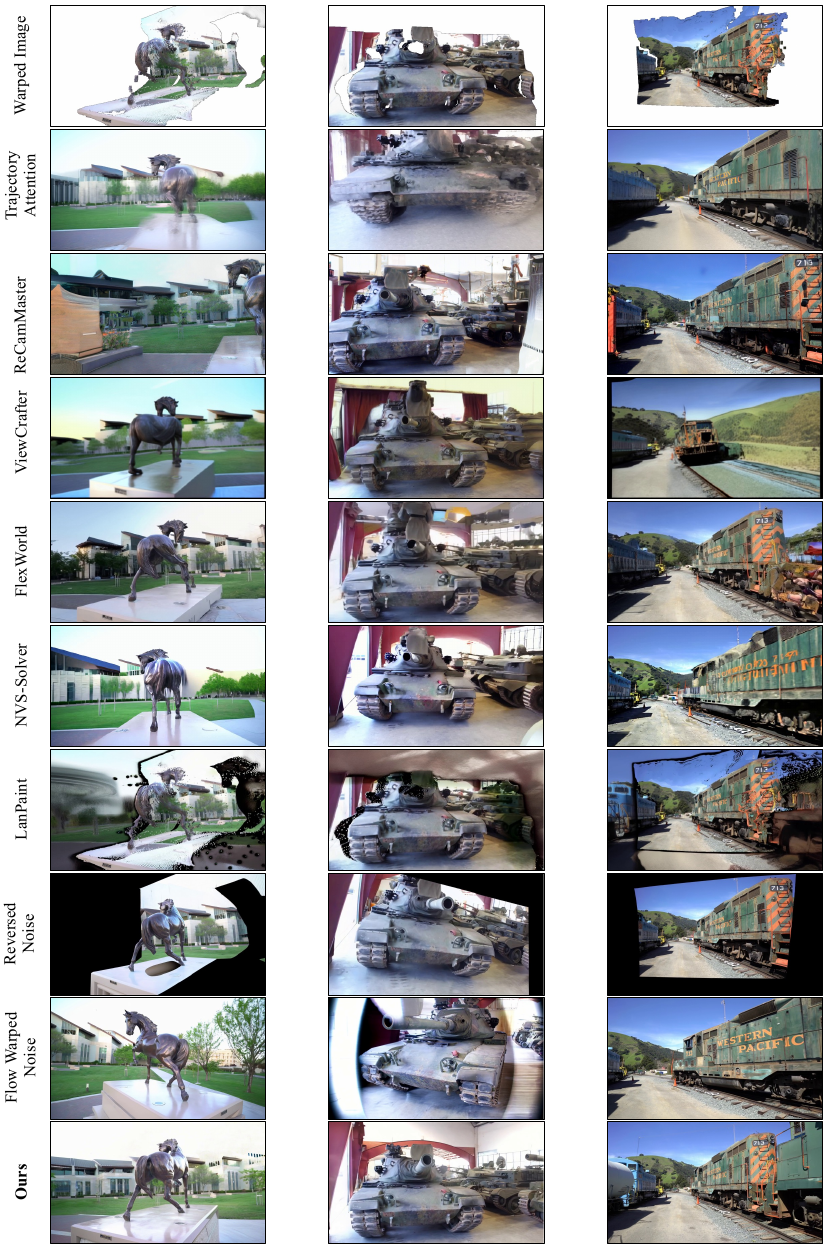}}
    \caption{Qualitative results of all baselines and our methods on Tanks-and-Temples dataset.}
    \label{fig:all-tnt}\vspace{-0.3cm}
\end{figure}

\begin{figure}[t]
\setlength{\abovecaptionskip}{ 2 pt}
\setlength{\belowcaptionskip}{ -2 pt}
\centering
\centerline{\includegraphics[width=1.\linewidth]{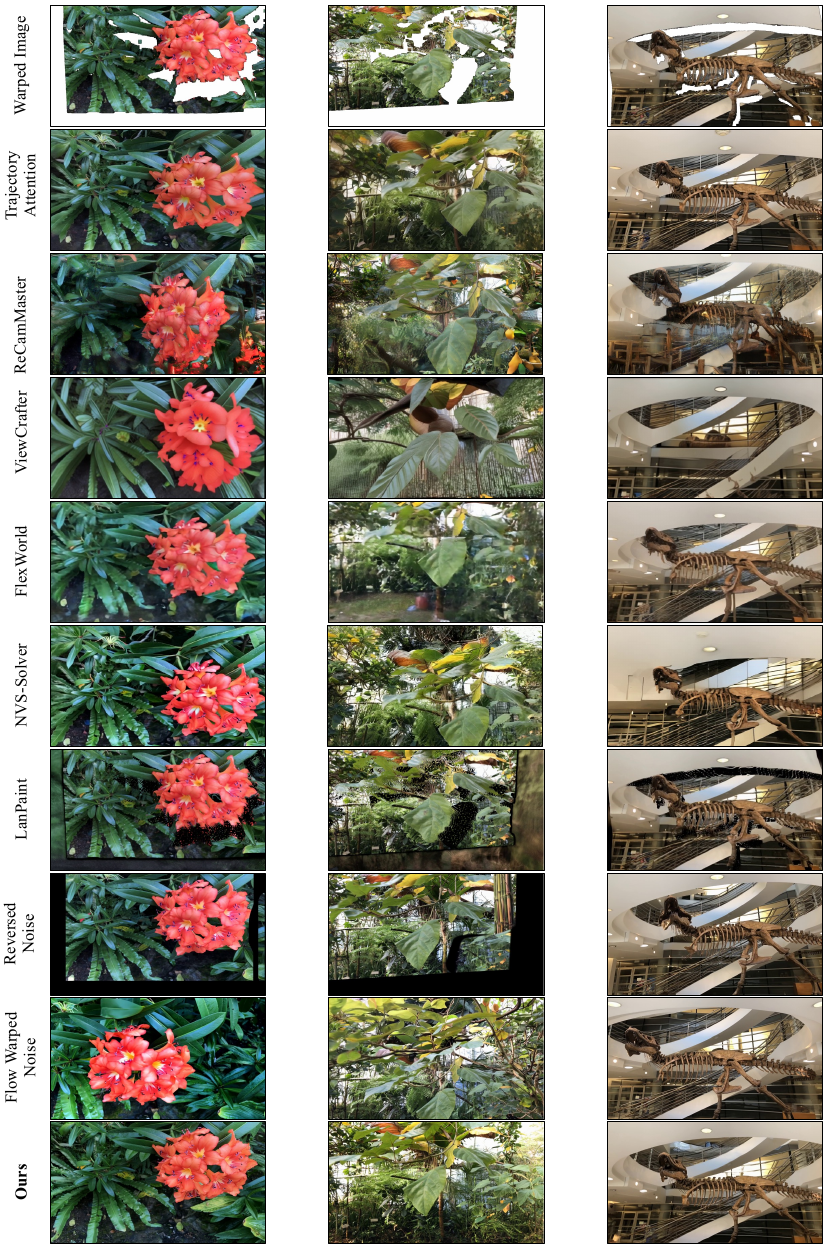}}
    \caption{Qualitative results of all baselines and our methods on LLFF dataset.}
    \label{fig:all-llff}\vspace{-0.3cm}
\end{figure}

\begin{figure}[t]
\setlength{\abovecaptionskip}{ 2 pt}
\setlength{\belowcaptionskip}{ -2 pt}
\centering
\centerline{\includegraphics[width=1.\linewidth]{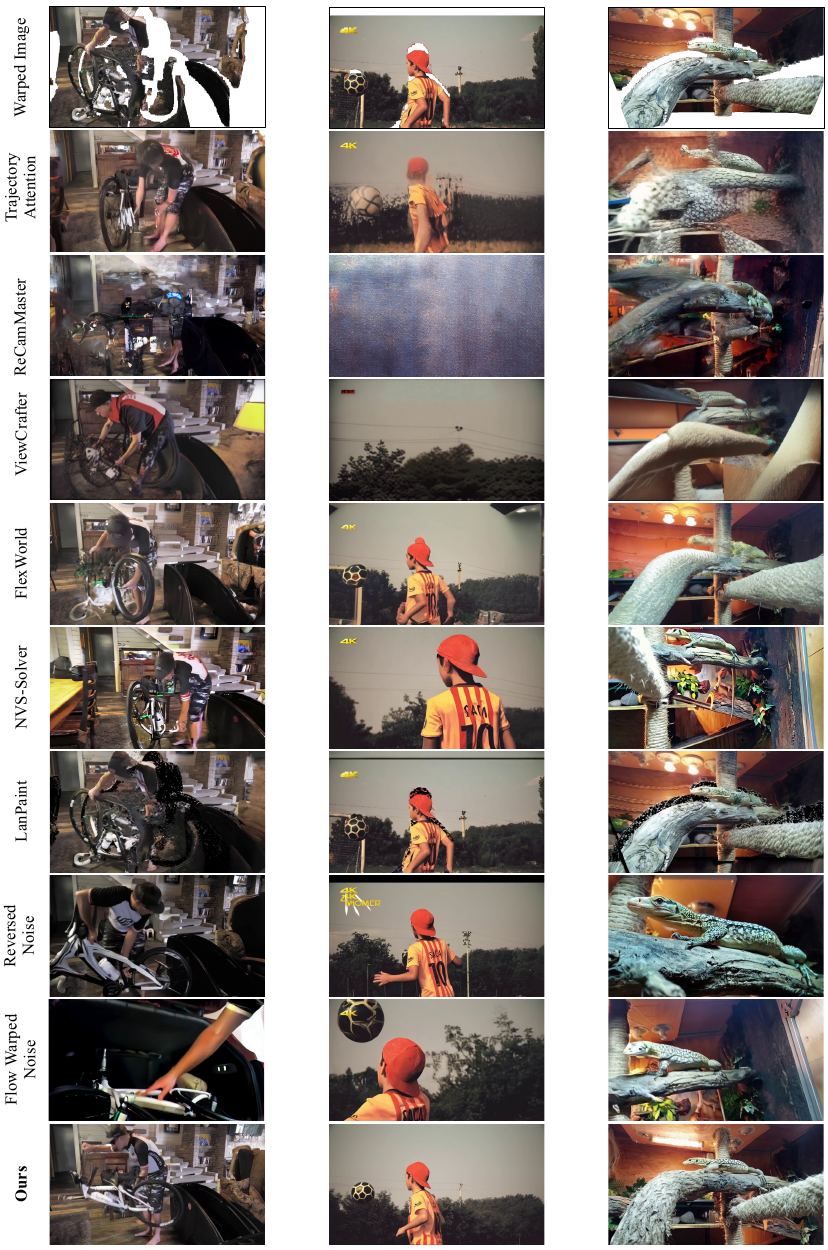}}
    \caption{Qualitative results of all baselines and our methods on DAVIS dataset.}
    \label{fig:all-davis}\vspace{-0.3cm}
\end{figure}

\end{document}